\documentclass[letterpaper]{article} % DO NOT CHANGE THIS
\usepackage[]{aaai25}  % DO NOT CHANGE THIS
\usepackage{times}  % DO NOT CHANGE THIS
\usepackage{helvet}  % DO NOT CHANGE THIS
\usepackage{courier}  % DO NOT CHANGE THIS
\usepackage[hyphens]{url}  % DO NOT CHANGE THIS
\usepackage{graphicx} % DO NOT CHANGE THIS
\usepackage{natbib}  % DO NOT CHANGE THIS AND DO NOT ADD ANY OPTIONS TO IT
\usepackage{caption} % DO NOT CHANGE THIS AND DO NOT ADD ANY OPTIONS TO IT
\usepackage{algorithm}
\usepackage{algorithmic}

\usepackage{latexsym}
\usepackage{amssymb}
\usepackage{amsmath}
\usepackage{amsthm}
\usepackage{mathrsfs}
\usepackage{booktabs}
\usepackage{enumitem}
\usepackage{color}

\newtheorem{theorem}{Theorem}
\newtheorem{lemma}[theorem]{Lemma}
\newtheorem{remark}[theorem]{Remark}
\newtheorem{corollary}[theorem]{Corollary}

\newtheorem{definition}[theorem]{Definition}

\usepackage{newfloat}
\usepackage{listings}
\DeclareCaptionStyle{ruled}{labelfont=normalfont,labelsep=colon,strut=off} % DO NOT CHANGE THIS
\floatstyle{ruled}
\newfloat{listing}{tb}{lst}{}
\floatname{listing}{Listing}

\newcounter{figtab}

\makeatletter
\let\oldfigure\figure
\let\endoldfigure\endfigure
\renewenvironment{figure}
  {\refstepcounter{figtab}\oldfigure}
  {\endoldfigure}

\let\oldtable\table
\let\endoldtable\endtable
\renewenvironment{table}
  {\refstepcounter{figtab}\oldtable}
  {\endoldtable}
\makeatother
 
\title{Accounting for Context: Shaping Moral Credences for Value Alignment}
\author{Jazon Szabo and Sanjay Modgil
}
\affiliations{King's College London, UK
}

\begin{document}

\maketitle

\begin{abstract}
Ensuring that agent behaviours are aligned with human moral values inevitably raises the problem of how to account for the plurality of moral perspectives that societies -- and even individuals -- typically adopt. Work on \textit{moral uncertainty} proposes mechanisms to fairly and democratically aggregate evaluations of actions across different moral theories. However, this paper argues that one needs to account for \emph{contextual factors} when aggregating moral evaluations. For example, consequentialist perspectives assume an ability to accurately determine how an agent's actions change the world; an assumption that often does not hold in real world settings. We, therefore, formalise agent decision making under moral uncertainty, while also accounting for these kinds of contextual factors. We thereby show that a seemingly commonsensical property -- the \emph{weak Pareto principle} -- is violated.  We argue that this apparent problem is, in fact, a variation of \emph{Simpson's paradox}, and hence reveals the limitations of aggregation mechanisms that ignore the impact of contextual factors.
\end{abstract}

\section{Introduction}\label{Sec:introduction}

\emph{Value alignment} aims at ensuring that the behaviours of intelligent (and ultimately `superintelligent') agents are aligned with human ethical values \cite{gabriel2020artificial}. However, human societies rarely exhibit moral consensus; societies (and even individuals) evaluate actions according to different ethical theories that may advocate conflicting prescriptions \cite{haidt2012righteous,awad2018moral,macaskill2020moral}. \emph{Pluralistic alignment} \cite{DBLP:conf/icml/SorensenMFGMRYJ24} (also  called \emph{democratic value alignment} \cite{gabriel2020artificial}) is therefore concerned with aligning agents' actions in a way that respects the plurality of ethical values found in human societies.

A notable approach to value alignment adopts a multi-step pipeline\footnote{We present a modified and simplified version here.} \cite{DBLP:journals/ibmrd/NoothigattuBMCM19} that: 1) elicits individuals' ethical preferences in simulated, ethically-salient scenarios (e.g. the Moral Machine experiment \cite{awad2018moral}); 2) uses some form of inverse reinforcement learning \cite{IRL} to derive the individuals' reward (and utility) functions from these preferences;  3) aggregates these functions so as to appropriately account for the diverse preferences, where the results of aggregation are adopted as the agent's reward function.

This paper focuses on the third step with implications for the whole pipeline; specifically, we focus on formalisms inspired by recent philosophical studies of `\emph{moral uncertainty}'\cite{macaskill2020moral}. The core idea is to extend the standard decision-theoretic framework to account for uncertainty about which ethical theories should be used to evaluate actions. This uncertainty is expressed in the form of \textit{moral credences} in various ethical theories\footnote{Extending the notion of credences (subjective probabilities) standardly associated with beliefs \cite{BeliefCredences} to credences in moral theories.}. For example, given ethical theories $t_1$ and $t_2$, each assigning numerical evaluations to actions $a$ and $b$, these evaluations are aggregated using a social choice framework\footnote{Note that within pluralistic alignment, \emph{social choice} is often used to formally study the process of aggregating different values \cite{baum2020social,DBLP:journals/corr/abs-2410-23953,DBLP:conf/icml/ConitzerFHHJ0MP24}.}  that accounts for the relative proportion of individuals endorsing each theory. Intuitively, these proportions are expressed as  moral credences\footnote{Indeed, these credences may express an \textit{individual's} relative confidence in the appropriateness/suitability of the theories' possibly distinct evaluations}. Consequently, moral uncertainty has been used in AI research to aggregate the values of different groups within societies \cite{DBLP:journals/mima/Bogosian17,bhargava2017autonomous,DBLP:conf/icml/EcoffetL21,martinho2021computer}. In summary then, moral uncertainty advocates explicit reference to the plurality of ethical theories that account for the diversity of elicited human evaluations in the above mentioned pipeline.

In this paper, we examine how contextual factors influence the moral credences assigned to ethical theories and, consequently, how accounting for these factors impacts the outcomes of moral aggregation. Consider an agent and the question of whether the agent should evaluation actions from a deontological or utilitarian perspective. Suppose a population is evenly split: $50\%$ prefer a deontological evaluation and $50\%$ a consequentialist one, resulting in equal credence ($0.5$) assigned to each theory. These evaluations are understood as measures of the actions' overall `goodness'. However, such credences are typically elicited in response to generic scenarios -- such as those used in standard trolley problems \cite{foot1967problem} or large-scale studies like the Moral Machine experiment \cite{awad2018moral} -- where context-specific details are abstracted away. In practice, it is neither realistic nor adequate to ignore these contextual factors. For instance, in a particular setting, the availability of computational resources for accurately estimating the consequences of actions should influence the weight (i.e. moral credence) assigned to a consequentialist evaluation. This suggests that moral credences are not static, but should be dynamically adjusted in light of relevant contextual constraints.

\paragraph{Running Example} 

Let us present our running example; we use this example to demonstrate many of our results regarding contextual features. Consider a small, mobile firefighter robot called FROBO whose main aim is to contain fires and facilitate rescue in environments inaccessible to humans. We suppose FROBO is deployed in a burning hospital, in a country that is under military blockade. FROBO has been sent to a hallway with two rooms, one on the left and one on the right, both of which contain  fires. FROBO only has the capacity to successfully deal with one of these fires. The left room  contains a significant amount of insulin; a medicine that is extremely hard to come by while the country is under blockade. Without this insulin, many diabetics may die \textit{if} the blockade is not lifted. On the other hand, there is an unconscious patient in the right hand room. If FROBO does not attend to the right room, the patient will die. What should FROBO do?

\begin{figure}
    \centering
    \includegraphics[width=1\linewidth]{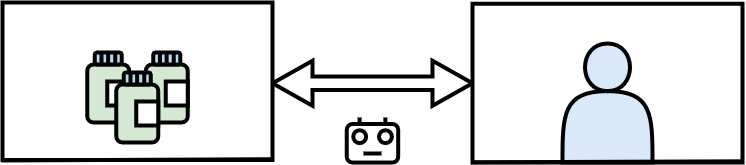}
    \caption{FROBO's dilemma: whether to save a person in immediate danger (right room) or preserve enough insulin to treat dozens of patients (left room)}
    \label{fig:FROBO}
\end{figure}

\paragraph{Contributions}
Section \ref{Sec:background} provides the necessary background for our formal framework as well as a short summary of deontological and utilitarian ethics. In Section \ref{Sec:ContextualFactors}, we review a set of contextual features that, we argue, should be considered when aggregating moral preferences derived from simulation-based studies. Sections \ref{Sec:Formalism} and  \ref{Sec:Technical} presents the paper’s core contributions. We begin (Section \ref{Sec:Formalism}) by formalising a model of moral uncertainty that explicitly incorporates these contextual factors. Then, in Section \ref{Sec:Technical} we demonstrate that a broad class of aggregation methods violates the \emph{weak Pareto principle}, a property often regarded as non-controversial. We interpret this result as a variant of the well-known \emph{Simpson’s paradox}, which we take to highlight the conceptual and practical limitations of context-insensitive aggregation. We argue that this insight motivates a new area of research aimed at understanding how contextual features shape aggregation of moral evaluations. Section \ref{Sec:Discussion} outlines key directions for future work. In particular, we suggest an alternative approach to addressing the limitations of our current three-step pipeline: the integration of `thicker’ ethical information, in particular through the use of structured dialogues and argumentation.

\section{Accounting for Contextual Factors  in Shaping Moral Credences}%\label{Sec:ContextualFactors}

\subsection{Background}\label{Sec:background}

This paper's contributions are presented within a framework proposed by \cite{macaskill2020moral} and formalised in \cite{szabo2024moral}. Let $T$ be a set of ethical theories and $A$ a set of actions. Each ethical theory $t \in T$ is a function $t: A \to \mathbb{R}$, so that for each action $a \in A$, $t(a)$ represents the \emph{choiceworthiness} of action $a$ with respect to theory $t$. The induced ordering $\preceq_t$ of a theory $t$ is the implicit ordering of actions by their choiceworthiness, i.e. for all $a, b \in A$, $a \preceq_t b$ iff $t(a) \le t(b)$. A \emph{credence function} $c$ is a function that maps each ethical theory $t \in T$ to a real number $c(t) \in (0, 1)$ such that the weights of all the theories sum up to $1$, i.e. $\sum_{t \in T}c(t) = 1$. Finally, %given a set of actions $A$,
an \emph{ethical framework} $\langle T, c \rangle$ is a pair consisting of a set of ethical theories $T$ and a credence function $c$.

This formalism facilitates the study of \emph{metanormative theories}: various aggregation functions that yield an overall ranking of actions by integrating the theories' choiceworthiness, each weighted by its associated credence $c(t)$. For example, in \cite{szabo2024moral}, the issue of which metanormative theories avoid fanatical outcomes is studied, where a fanatical outcome arises when a theory $t$ with low credence/advocated by very few in a society, has an outsized effect upon aggregation.
%Could add that these theories and actions are drawn from an infinite set of potential actions and ethical theories but the agent only has a finite subset of theories and actions

In this paper we focus on two ethical theories: \emph{deontology} and \emph{utilitarianism}. Utilitarianism is a consequentialist ethical theory \cite{sep-consequentialism}: the right thing to do is the action whose consequences impartially maximise well-being. On the other hand, deontological theories \cite{sep-ethics-deontological} specify rules encoding obligations and prohibitions on agent behaviours, without necessarily accounting for the consequences of abiding by these rules (hence the deontological maxim `ends do not justify the means'). We also consider the \emph{two-level utilitarianism} variant of utilitarianism \cite{DBLP:journals/ais/Bauer20}; a kind of synthesis of utilitarianism and deontology. Normally, ethical rules guide agent behaviours. However, in novel or unforeseen situations -- where no applicable rules exist or the existing ones fall short --agents revert to direct consequentialist calculations to determine the morally appropriate course of action

\subsection{The Impact of Contextual Factors}\label{Sec:ContextualFactors}

Recall (Section \ref{Sec:introduction}) that we assume that moral credences are a proxy for the relative proportions of individuals in a group that advocate for the use of a theory $t$'s choiceworthiness. We do not consider here the issue of how the exact choiceworthiness values are assigned (although we comment further on this issue in Section \ref{Sec:Discussion}). Suffice it to say that, as discussed in \cite{macaskill2020moral}, different theories posit different moral stakes associated with actions. For example, the deontological prohibition on torture is absolute and thus may assign a negative choiceworthiness that is an order of magnitude greater than the negative choiceworthiness assigned by consequentialism (as in the archetypal dilemma in which torture might be used on a terrorist to locate a ticking bomb).

We also assume that the credence-adjusted aggregation that is used to align the utility function of an AI agent, is obtained on simulated scenarios, and so does not account for the specific contexts in which the AI agent is deployed (although Section \ref{Sec:Discussion} will consider examples where \textit{an individual's} relative credences in distinct ethical theories is used to support AI decision making `on the fly').

We are therefore interested in accounting for how context-specific factors can be used, on the fly, to adjust the moral credences in the relevant ethical theories that are elicited on simulations, and the impact this has on the credence adjusted aggregations. We now briefly review a non-exhaustive list of context-specific factors:

\paragraph{Resource Bounds}
A crucial insight from the study of \emph{machine ethics} is that different ethical theories necessitate different algorithmic implementations \cite{bench2020ethical}. 
In this paper we assume (as in for example \cite{DBLP:journals/ais/Bauer20}) that rule-based ethics  such as deontology  can be implemented efficiently\footnote{For example, if deontological rules are specified in \emph{propositional defeasible logic} then inference has linear time complexity \cite{MAHER_2001}.} (although this is a matter of contention \cite{stenseke2024computational}).  In contrast, we assume that consequentialist ethics, such as utilitarianism, cannot be implemented efficiently, for three reasons \cite{stenseke2024computational}: 1) there is a combinatorial explosion in the number of possible actions that an agent has to evaluate; 2) causal predictions are expensive under epistemic uncertainty; 3) causal predictions are expensive in the presence of other agents. While Reinforcement Learning (RL) methods can alleviate these issues, such stochastic methods presuppose a trial-and-error method, which can be disastrous in morally-salient situations.

With this in mind, the essential idea is that under resource bounds some ethical theories fare well, while others do not. For example, assume for simplicity that FROBO has a single deontological ethical rule: fight fires in rooms where someone is in immediate danger. In our running example, the rule applies only to  the right room, and so deontically prescribes a preference for FROBO attending to the right room.
%In other words, FROBO's implementation of deontology is reliable in this situation: FROBO can accurately simulated what deontology `actually' proscribes.

Consequentialist ethical theories do not fare well under resource bounds. FROBO has limited time; if FROBO waits too long, both rooms will burn down. FROBO also has limited computational resources for reasoning; FROBO is a small robot, not a supercomputer. Suppose then that under the given bounds, and given the available information, FROBO predicts that the military blockade will stay in place and that insulin supplies will remain scarce for the forseeable future. Hence, FROBO calculates that saving the medicine, not the patient, is the action which leads to the most number of people saved. Therefore, FROBO attends to the left room.

%Therefore, FROBO calculates that saving the medicine, not the patient, is the action which leads to the most number of people saved; FROBO predicts that the military blockade will stay in place. Therefore, FROBO attends to the left room.

However, such predicted outcomes can easily be wrong.  FROBO's resource-bounded processing of the available information indicates that ongoing negotiations are unlikely to end soon. However, a more thorough going analysis would have indicated
that the military blockade would likely end relatively soon. FROBO reached the wrong conclusion, but would have reached the right conclusion if its reasoning was subject to less strict bounds.  As a result, FROBO's implementation of utilitarianism is unreliable \textit{in the given resource-bounded context}. On the other hand, whether the blockade will continue or not makes no difference to the deontological obligation to  attend to the right room. Consequently, under resource bounds, not every theory is equally appropriate. The broader lesson demonstrated by this scenario is that  when aggregating the evaluations of ethical theories, it does not suffice to rely on their relative support (moral credences); one needs to account for the extent to which concrete context-specific factors differentially impact the appropriateness of their application. In the case of FROBO, even if a majority of people prefer utilitarianism to deontology, deontology's evaluation should take precedence.

\paragraph{Uncertainty}
Another closely related (to resource bounds) contextual factor is uncertainty about: i) the facts; ii) whether an action will have the desired (intended consequences); iii) the consequences \textit{tout court}. While it is unrealistic to comprehensively account for uncertainty in simulated scenarios, they can significantly influence moral evaluation. Consider a case where deontology prescribes action $a$, whereas utilitarianism advocates action $b$, based on the immediate consequences of the actions. While uncertainty of types i) and ii) may not differentiate between deontology and utilitarianism, the extent to which there is uncertainty of type iii) -- broader or more downstream consequences -- may well impact the moral credence assigned to utilitarianism.

A classic example is the trolley dilemmas \cite{foot1967problem}. A common objection that subjects raise when asked to judge whether to steer the trolley so as to kill the one, rather than let the trolley continue on its way and kill the five, is that nothing is assumed known about the five or the one \cite{GreeneMoralTribes}. What if the five were escaped murderers from a local prison, and the one a renown scientist on the verge of a major medically beneficial discovery? This is an admittedly extreme and unrealistic scenario, albeit one that makes an intuitive point: the more the uncertainty w.r.t. downstream consequences, the less credence one might intuitively assign to consequentialist evaluations\footnote{Note that research in moral psychology supports the view that as uncertainty increases, the appropriateness of relying on utilitarian calculations diminishes \cite{kortenk2014ethics}.}.

\paragraph{Novelty}
\emph{Rule utilitarianism} \cite{sep-consequentialism} is, as the name implies, a rule-based variant of utilitarianism (the version of utilitarianism we have encountered thus far is called \emph{act utilitarianism}). Under this view, agents ought to act according to moral rules derived from utilitarian calculations -- rules that, when generally followed, promote the greatest overall well-being in the long run. Such adherence serves two key purposes: it avoids the need for computationally costly utility calculations on a case-by-case basis \cite{DBLP:journals/ais/Bauer20}, and it facilitates social coordination \cite{DBLP:conf/aies/SerramiaLRMWA18}.

However, the rules may be imperfect: in novel or unexpected situations, strict adherence to the rules may result in undesirable outcomes \cite{BenchMod,bench2019norms}. For example, consider the famous `no vehicle in the park' example \cite{schlag1999no}. There is a rule that states that no vehicles are to drive into the park, as their presence can endanger those on foot. Now, consider that there is an emergency and an ambulance is needed to pick up an injured person. Clearly, the rule prohibiting the ambulance's entry did not anticipate this situation; the ambulance should enter the park. Indeed, as some philosophers argue \cite{hare1963freedom}, in such novel situations, agents ought to reason from first principles. An ambulance driver ought to realise that the rule is mistaken and should be violated. Therefore, some rule-based ethical theories, such as rule utilitarianism, perform poorly in unexpected, novel situations. In such situations, their moral credence ought to be lowered.

%What we loosely term `second order uncertainty' refers to the extent to which   a choice of
%action is morally salient in a given context. The same action may have different moral significance in different contexts. 

Note that these concerns also extend to deontological rules that are not grounded in consequentialist considerations. Consider the following example. A literal minded interpretation of one of Asimov's famous laws of robotics \cite{asimov1940robot} -- \textit{a robot may not injure a human being or, through inaction, allow a human being to come to harm} --  might entail a rescue robot intervening in a situation where a citizen is performing an emergency tracheostomy on a victim in  a disaster zone. Contextual information should clearly reduce credence in this deontological prohibition.

%Another non-credence consideration: value of moral information increases, so agent may need to first act to gather such information (refer to William MacAskill's work)
%\begin{itemize}
%\item First Order Uncertainty. Uncertainty about i) facts ii) whether an action will have the desired (intended consequences) iii) the consequences tout court. i) and ii) relevant for deontology and utilitarianism. iii) relevant for 
%utilitarianism. Remind reader that these are contextual factors that are impractical to account for in 
%the simulated scenarios, even though in classic decision theory one does/can account for them in the moral evaluation (choiceworthiness) of actions.  

% \paragraph{Risk} Risk neutral, Risk Averse, Lara Buchak . See \url{https://ndpr.nd.edu/reviews/risk-and-rationality/} for summary of Buchak's work. 

\paragraph{Supererogation}
\emph{Supererogation} \cite{sep-supererogation} is the phenomenon wherein an action is good and yet not morally required. For example, jumping into a stormy and treacherous sea to save a drowning person is clearly a morally good action and yet goes beyond what is reasonably expected from anyone. We do not typically expect strangers to risk their own lives to save someone else's.

The extent to which an action is considered supererogatory may depend on alternatives courses of action that were not anticipated in simulations. Again, consider the Moral Machine experiment \cite{awad2018moral}, and suppose that on simulations in which an autonomous vehicle (AV) is in a one-way tunnel, subjects judge that the AV should swerve sharply to the right ($r$) to avoid ploughing straight on ($s$) and killing three pedestrians. However, in so doing $r$ is likely to lead to the death of the passenger as the simulation is such that the car will crash head on into the tunnel wall.

However, suppose that in a real-world deployment of the AV, the width of the tunnel is such that the AV can swerve slightly to the right and slide along the inside wall of the tunnel ($w$), with the non-negligible possibility of injury to the passenger and serious injury to the one pedestrian closest to the tunnel wall (a scenario unanticipated in simulations). Now, in \textit{this} context, $r$ remains the preferred utilitarian option, but it has become supererogatory in a way that it was not in the simulated scenario. The presence of the unanticipated third option $w$ -- which may lead to greater overall harm, but with less severe consequences for the passenger -- introduces a morally acceptable alternative that renders $r$ no longer strictly required. Arguably, this expanded context should lower moral credence in the utilitarian prescription of $r$, as it shows that the original obligation was contingent on a limited set of actions.

 %say that ethical theories which reason from first principles are similarly, arguably, better suited for novel situations
%On the other hand, ethical theories, which reason from first principles, such as act utilitarianism, should maintain their moral credence.
\section{Formalising Contextual Shaping of Moral Credences} \label{Sec:Formalism}
In this section, our formalisation of ethical decision-making accounts for the contextual shaping of moral credences. Then in Section \ref{Sec:Technical} we show that a class of intuitively reasonable metanormative theories -- including Maximum Expected Choiceworthiness (MEC) -- violate the weak Pareto principle. However, we argue that this is not necessarily undesirable but rather a variation on Simpson's paradox.

In the FROBO example, the left room contains insulin, which may save numerous lives if the military blockade continues. However, as mentioned in Section \ref{Sec:ContextualFactors}, due to resource bounds, FROBO's utilitarian evaluation about the left room is potentially unreliable. On the other hand, FROBO can accurately evaluate the consequences of attending the right room; it would save a single person. Therefore, we would like to modulate the credences such that for the left room, utilitarianism has a lower credence, while for the right room, its credence remains unchanged. To accommodate such context-dependent and action-specific credences, we define \emph{credence profiles}.

\begin{definition}[\textbf{Credence profile}]\label{Def:CredenceProfile}
Given actions $A = \{a,b,\ldots\}$ and a set of ethical theories $T$,
%and a credence function $c$ for a theory $t \in T$,
a \emph{credence profile} $C = \langle c_a, c_b,  ... \rangle$ is a list of credence functions, one for each action. 
\end{definition}

For example, $c_r$ denotes the credence function for action $r$, attending to the right room.

\begin{remark}\label{ActionSpecificCredences}
Note that in Section \ref{Sec:background} we have specified a \textit{generic} function $c$ that maps a given ethical theory $t$ to its moral credence. In the above definition, a credence profile includes credence functions \textit{individuated with respect to each action}. As formalised later in Definitions \ref{Def:prof} and \ref{DefMini}, this will allow contextual factors (`features') to differentially adjust a theory's moral credence as it pertains to a given action, rather than assuming the credence in a theory's evaluation of actions to be uniformly adjusted by the context.
\end{remark}

We  formalise \emph{contextual features} as functions that affect credences as they pertain to actions:

\begin{definition}[\textbf{Contextual feature}]\label{Def:ContextualFeature}
Given a set of actions $A$ and a set of ethical theories $T$, let a \emph{contextual feature} be a function $g: A \times T \to (0, 1]$.  
\end{definition}

Let $rb$ denote the resource bounds contextual feature. In our FROBO example, utilitarianism $\mathbf{u}$ does not produce a highly reliable evaluation for the left room under the resource bounds that apply in the given context, and so $rb(l, \mathbf{u}) = 0.1$ is small. In general, we want lower-valued contextual features to decrease (or at least not increase) the credence of the theory. Hence, we expect the credence $c_\mathbf{u}(l)$ to be decreased. On the other hand, deontology $\mathbf{d}$'s evaluation of the action $l$ is not impacted by the bounds on resources, and so $rb(l, \mathbf{d}) = 1$. Similarly, we typically expect higher-valued contextual features to increase (or at least not  decrease) the credence of the theory\footnote{One might expect that if $rb(l, \mathbf{d}) = 1$, this would have no impact on the moral credence assigned to $\mathbf{d}$. However, because the total moral credence across ethical theories must sum to $1$, a reduction in credence for one theory due to a contextual factor may require a corresponding increase in credence for another theory that is not similarly affected}. Here, sice $rb(l, \mathbf{d})$ is maximal, we expect the credence $c_\mathbf{d}(l)$ to be increased. 

Consider the right room $r$. As discussed in Section \ref{Sec:ContextualFactors}, utilitarianism $\mathbf{u}$ produces reliable evaluations for the right room, even under the resource bounds, and so $rb(r, \mathbf{u}) = 1$ is maximal. Similarly, deontology $\mathbf{d}$'s evaluation of $r$ is also not impacted by resource bounds, and so $rb(r, \mathbf{d}) = 1$ is also maximal.
%TODO add a footnote saying that credences normalise, that's why they can increase
% For the right room, the confidence of utilitarianism and deontology would be both high, so $g(r, u) = g(r, d) = 1 - \nu$ for some small $\nu$. Therefore, in this situation, we expect the credences to be unaffected by confidence, as they are the same. That is, we expect $c_r(u) = 0.9$ and $c_r(d) = 0.1$ . %TODO update example
(Notice that we do not allow contextual features to have a $0$ evaluation; we do not want that a contextual feature completely invalidates credence in an ethical theory.)
%On the other hand, we allow contextual features to have a $1$ evaluation, as in some cases  contextual features do not impact credence in a theory
%contextual features naturally have a maximal value, which can be achieved. Consider resource bounds $rb$ and deontology $d$; because $d$ is not affected by the resource bounds for action $l$, we have $rb(l, d) = 1$.

In Section \ref{Sec:ContextualFactors}, we reviewed a number of different contextual features. Since we want our formalism to be sufficiently general to accommodate a range of contextual factors, we define a \emph{context} to be an $m$-ary list of contextual features.

\begin{definition} [\textbf{Context}]
Given a set of actions $A$, an $m$-ary \emph{context} $\mathit{con}$ is a list of contextual features $\langle g_1, ..., g_m \rangle$ such that for all $l \in [1, m]$,  $g_l: A \times T \to (0, 1]$ is a contextual feature.
\end{definition}

For example, $\langle rb \rangle$ denotes a context containing only $rb$.

%We denote by $\mathit{con}_g$ the unary context consisting only of the contextual features $g$, i.e. $\mathit{con}_g = \langle g \rangle$.

There are different options, which we encode as \emph{adjustment functions},  as to how to formalise updates to moral credences on the basis of contextual features.  So, an adjustment function takes the initial moral credences (credence function $c$) and updates them, taking into account the contextual features enumerated in the context $\mathit{con}$.
 
\begin{definition}[\textbf{Adjustment function}]
Given a set of actions $A$, $c$  a credence function and $\mathit{con}$  an \emph{adjustment function}, $\mathit{adj}$  maps   $\langle A, c, \mathit{con}\rangle$ to a credence profile $C$.
\end{definition}

We use the notation $\mathit{adj}(A, c, con)$ to denote application of $\mathit{adj}$ to   $\langle A, c, \mathit{con} \rangle$.
%TODO say that we can derive the set of theories implicitly? Probably unnecessary
%TODO emphasise that the different credence functions in the credence profile don't have to be equal? maybe demonstrate this in the example of either prod or mini
In this paper, we introduce two adjustment functions. The first -- $\mathit{prod}$ --  takes the product of all the contextual features and the initial credence function. In other words, each contextual reason can multiplicatively increase the credence or decrease it.

\begin{definition}
[\textbf{Adjustment function   \textit{prod}}]\label{Def:prof}
Let $T$ be set of ethical theories, $A$ a set of actions $\{a,b,\dots\}$, and  $\mathit{con}$ the context $\langle g_1, ..., g_m \rangle$.\\ Then, for every function $c$ that assigns a moral credence to some $t \in T$, $prod$ is function that returns the context adjusted moral credence of $t$'s evaluation of each action in $A$ (i.e., a credence profile  \footnote{Recall Definition \ref{Def:CredenceProfile} and Remark \ref{ActionSpecificCredences}.}).\\ That is to say, for a given theory $t$, $C$ is the credence profile $\langle c_a, c_b... \rangle =  \mathit{prod}(A, c, \mathit{con})$, where for $x \in A$:

    %Let $\mathit{prod}$ be an adjustment function such that, given a set of ethical theories $T$, for every credence function $c$ and context $\mathit{con} = \langle g_1, ..., g_m \rangle$, $\mathit{prod}$ is defined where for all actions $a$ and theories $t \in T$, (letting $C = \mathit{prod}(c, \mathit{con})$ with $C = \langle c_a, c_b... \rangle$)
\[c_x(t) = \alpha \left[\prod_{l \in [1, m]}g_l(x, t)\right]c(t)\]
and where $\alpha \in \mathbb{R}^+$ is a normalising constant defined such that $\sum_{t \in T}c_x(t) = 1$. (Recall that the credence functions have to be unitary, i.e., sum up to $1$).
\end{definition}

Consider FROBO example with initial credences $c(\mathbf{u}) = 0.6$ and $c(\mathbf{d}) = 0.4$; the initial credence of utilitarianism is higher than deontology. Here, the only contextual feature we take into account is resource bounds $rb$, where $rb(l, \mathbf{u}) = 0.1$ and $rb(l, \mathbf{d}) = 1$. Then the updated (context-adjusted) credences are $c_l(\mathbf{u}) = \alpha \times 0.1 \times 0.6 = 0.06 \alpha$ and $c_l(\mathbf{d}) = \alpha \times 1 \times 0.4 = 0.4\alpha$. Since $c_l$ is a credence function, we must have $0.06\alpha + 0.4\alpha = 1$ and so $\alpha = \frac{50}{23}$. Therefore, $c_l(\mathbf{u}) \approx 0.13$ and $c_l(\mathbf{d}) \approx 0.87$. In other words, despite the initial credence favouring utilitarianism, the adjusted credence for the left room strongly prefers deontology.

For the right room, we have $rb(r, \mathbf{u}) = 1$ and $rb(r, \mathbf{d}) = 1$, as explained earlier (following Definition \ref{Def:ContextualFeature}). Then, the updated credences for the right room is $c_r(\mathbf{u}) = \beta \times 1 \times 0.6 = 0.6\beta$ and $c_r(\mathbf{d}) = \beta \times 1 \times 0.4 = 0.4 \beta$. Since $c_r$ is a credence function, we must have $0.6 \beta + 0.4 \beta = 1$ and so $\beta = 1$. Therefore, $c_r(\mathbf{u}) = 0.6$ and $c_r(\mathbf{d}) = 0.4$. In other words, the adjusted credences for the right room match the initial credence $c$, unlike the adjusted credences for the left room (recall that the adjusted credence functions are individuated with respect to the different actions).

%TODO say that Prod implicitly assumes that for each g_l,  g_l(a, t) != 0 is true for some social alternative x (otherwise, weights cannot add up to 1)

The second adjustment function we introduce is $\mathit{mini}$, which takes the minimum of all contextual features and the initial credence. In other words, our updated (non-normalised) credence should not exceed the credence of any contextual feature. Intuitively, this is a `risk-averse' way of  adusting credences as low contextual features/moral credence cannot be traded off by high contextual features/moral credences. By contrast, $\mathit{prod}$ can be said to be `risk-neutral' as the low-values of contextual features can be negated by the high-values of other contextual features.

\begin{definition}[\textbf{Adjustment function \textit{mini}}]\label{DefMini} 
Let $T$ be set of ethical theories, $A$ a set of actions $\{a,b,\dots\}$, and  $\mathit{con}$ the context $\langle g_1, ..., g_m \rangle$.\\ Then, for every function $c$ that assigns a moral credence to some $t \in T$, $mini$ is function that returns the context adjusted moral credence of $t$'s evaluation of each action in $A$.\\ That is to say, for a given theory $t$, $C$ is the credence profile $\langle c_a, c_b... \rangle =  \mathit{mini}(A, c, \mathit{con})$, where for $x \in A$:

%Let $\mathit{mini}$ be an adjustment function such that, given a set of ethical theories $T$, for every credence function $c$ and context $\mathit{con} = \langle g_1, ..., g_m \rangle$, $\mathit{mini}$ is defined where for all actions $a$ and theories $t \in T$, (letting $C = \mathit{mini}(c, \mathit{con})$ with $C = \langle c_a, c_b... \rangle$)
\[c_x(t) = \alpha \min\left(\left[\min_{l \in [1, m]}g_l(x, t)\right], c(t)\right)\]
where $\alpha \in \mathbb{R}^+$ is a normalising constant defined such that $\sum_{t \in T}c_x(t) = 1$. (Recall the unitary requirement).
\end{definition}

% Consider a unitary $G = \langle g \rangle$. In such situations, the above equation simplifies to $w_x'(i) = \alpha \min(g(a, t), c_a(t)$).
%Mini implicitly assumes that for each g_l,  g_l(a, t) != 0 is true for some social alternative x (otherwise, weights cannot add up to 1)

Consider again the FROBO example with initial credences $c(\mathbf{u}) = 0.6$ and $c(\mathbf{d}) = 0.4$. Again, we only  take into account   resource bounds ($rb$) where $rb(l, \mathbf{u}) = 0.1$ and $rb(l, \mathbf{d}) = 1$. Then the updated credences for the left room is $c_l(\mathbf{u}) = \alpha \min(0.1,  0.6) = 0.1 \alpha$ and $c_l(\mathbf{d}) = \alpha \min(1, 0.4) = 0.4\alpha$. Since $c_l$ is a credence function, $0.1\alpha + 0.4\alpha = 1$ and so $\alpha = 2$. Therefore, $c_l(\mathbf{u}) = 0.2$ and $c_l(\mathbf{d}) = 0.8$. In other words, despite the initial credence favouring utilitarianism, the adjusted credence for the left room strongly prefers deontology.

For the right room: $rb(r, \mathbf{u}) = 1$ and $rb(r, \mathbf{d}) = 1$. Then, the updated credences for the right room are $c_r(\mathbf{u}) = \beta \min(1, 0.6) = 0.6\beta$ and $c_r(\mathbf{d}) = \beta \min(1, 0.4) = 0.4 \beta$. Since $c_r$ is a credence function: $0.6 \beta + 0.4 \beta = 1$ and so $\beta = 1$. Therefore, $c_r(\mathbf{u}) = 0.6$ and $c_r(\mathbf{d}) = 0.4$. That is, the adjusted credences for the right room match the initial credence $c$, unlike the adjusted credences for the left room.

In the moral uncertainty literature, \emph{metanormative theories} tell us how to order the actions, given an ethical framework $\langle T, c \rangle$  consisting of a set of ethical theories $T$, a credence function $c$, and the evaluations assigned by the theories to the actions.
 Since we are interested in context adjusted credences individuated w.r.t. the actions, in this paper a metanormative theory $f$ takes a credence profile as an argument, not a credence function.

\begin{definition}[\textbf{Metanormative theory}]
\label{Def:MetanormativeTheory}A \emph{metanormative theory} $f$ is a function that maps every set of ethical theories $T$ and credence profile $C$ to a total preorder $f(T, C)$ over the set of actions.
\end{definition}

In this paper, we use MEC as our primary metanormative theory\footnote{Note that this is because MEC is the standard in the moral uncertainty literature \cite{macaskill2016normative,DBLP:journals/mima/Bogosian17}.}, which we now formally define. Given a set of ethical theories $T$ and a credence profile $C$, the weighted arithmetic mean of an action $x \in A$ is defined as:
\[\mathit{wam}(T, C, x) = \sum_{t \in T}c_x(t)t(x)\]

\begin{definition}[\textbf{Maximising expected choiceworthiness}\label{def:fan:mec}]  ($\textit{mec}$)]
 Let $a, b \in A$ be actions, and $\langle T, C \rangle$   an ethical framework. Then $a \preceq_\textit{mec} b$ iff $\mathit{wam}(T, C, a) \le \mathit{wam}(T, C, b)$, where $\textit{mec}(T, C)~=~\preceq_\textit{mec}$.
\end{definition}

Note that the weighted arithmetic mean $\mathit{wam}$ is often called the \emph{expected choiceworthiness} \cite{macaskill2014normative}, hence the name `maximising expected choiceworthiness'.

\begin{table}[t]
    \centering
    \begin{tabular}{ccccc}
         \hline
         & \multicolumn{2}{c}{\textbf{Left room} $l$} & \multicolumn{2}{c}{\textbf{Right room} $r$} \\
         & $c_l(t)$& $t(l)$& $c_r(t)$& $t(r)$\\ \hline
         Utilitarianism $\mathbf{u}$& $0.2$& $1$& $0.6$& $0$\\
         Deontology $\mathbf{d}$& $0.8$& $0$& $0.4$& $3$\\ \hline
         $\mathit{wam}$& \multicolumn{2}{c}{$0.2$} & \multicolumn{2}{c}{$ \underline{1.2}$}\\
    \end{tabular}
    \caption{FROBO's evaluations and credence profile. Here, the columns $c_l(t)$ and $c_r(t)$ denote the credence profile of the ethical theories. Moreover, the columns $t(l)$ and $t(r)$ denote the evaluations of the actions -- left room and right room, respectively -- by the different ethical theories. Finally, $\mathit{wam}$ gives the expected choiceworthiness of the actions.} %Good enough explanation?
    \label{table:MEC}
\end{table}

Consider  FROBO  with adjustment function $\mathit{mini}$ and a context  consisting only of $rb$ (see Table \ref{table:MEC}). Recall our earlier calculations: $c_l(\mathbf{u}) = 0.2$, $c_l(\mathbf{d}) = 0.8$, and $c_r(\mathbf{u}) = 0.6$, $c_r(\mathbf{d}) = 0.4$. Moreover, assume that the evaluations are $\mathbf{u}(l) = 1$, $\mathbf{u}(r) = 0$, and $\mathbf{d}(l) = 0$, $\mathbf{d}(r) = 3$. In other words, while utilitarian calculations prefer the left room, deontological rules prioritise the right room. MEC orders actions according to their expected choiceworthiness $\mathit{wam}$. For the left room, we have: 
\[\mathit{wam}(\{\mathbf{d}, \mathbf{u}\}, \langle c_l, c_r \rangle, l) = 0.2 \times 1 + 0.9 \times 0 = 0.2\]
Similarly, for the right room, we have: 
\[\mathit{wam}(\{\mathbf{d}, \mathbf{u}\}, \langle c_l, c_r \rangle, r) = 0.6 \times 0 + 0.4 \times 3 = 1.2\]
Thus, due to its higher expected choiceworthiness, MEC leads FROBO to choose the right room.

\section{Violation of the Weak Pareto principle}\label{Sec:Technical}
We now define an important basic property of moral uncertainty (and social choice): the weak Pareto principle, which ensures that unanimous decisions are respected.

\begin{definition} [\textbf{Weak Pareto principle}] \label{def:weakPareto}
A metanormative theory $f$ and an adjustment function $\mathit{adj}$ are said to satisfy the \emph{weak Pareto principle}, if for every ethical framework $\langle T, c \rangle$, every context $\mathit{con}$, every pair of actions $a, b \in A$,
\begin{itemize}
\item if for every theory $t \in T$ $t(a) < t(b)$ holds
    \item $a \prec_{f} b$ must hold, where
    \item ${\preceq_f\ = f(T, C)}$ and $C = \mathit{adj}(C, \mathit{con})$.
\end{itemize}
\end{definition}

%Here, the context matters because we are saying that unanimity should be respected regardless of what the context may be.

We identify a class of metanormative theories and adjustment functions -- specifically \emph{inter-theoretically responsive metanormative theories} and \emph{context-surjective adjustment functions} --
 that jointly \textit{violate} the weak Pareto property. 

\begin{definition} [\textbf{Inter-theoretic responsiveness}] \label{def:intertheoretic}
A metanormative theory $f$ is said to be \emph{inter-theoretically responsive} if for all sets of ethical theories $T$ and all pairs of actions $a, b \in A$, where $u(a) > v(b)$ holds for some theories $u, v \in T$, there exists credence profile $C$ such that $a \succ_f b$, where $\preceq_f\ = f(T, C)$.
\end{definition}
%This idea is closely related to fanaticism; I think every fanatical metanormative theory also satisfies the above property. On the other hand, I don't think HM satisfies this property
%It's also related to responsiveness but not the same

In other words, an inter-theoretically responsive metanormative theory is such that if there is a justification for preferring $a$ over $b$ -- specifically, if there exists two theories $u$ and $v$ such that $u(a) > v(b)$ -- then there must exist a credence profile $C$ such that $a$ is preferred to $b$ in the aggregate ordering, i.e. $a \succ_f b$ holds.
% At a first glance, having $u(a) > v(b)$ for some theories $u, v$ is not sufficient reason to prefer $a$ over $b$. However, as we argue later in Section \ref{sec:simpson}, under some assumptions, this is, in fact, enough reason to prefer $a$ over $b$. An important example of an inter-theoretically responsive metanormative theory is MEC (see Theorem \ref{thm:mec-intertheoretic}).

\begin{definition} [\textbf{Context surjectivity}] \label{def:surjectivity}
An adjustment function $\mathit{adj}$ is \emph{context-surjective} if for any set of actions $A$,  credence function $c$ and credence profile $C$, there exists a context $\mathit{con}$ such that $\mathit{adj}(A, c, \mathit{con}) = C$.
\end{definition}
Alternatively, a function $\mathit{adj}$ is context-surjective if function $\mathit{adj}(A, c, \cdot)$ is surjective for all credence functions $c$.

Intuitively, context surjectivity means that the context can update the initial credence function arbitrarily. Both $\mathit{prod}$ and $\mathit{mini}$ are context-surjective (see Theorems \ref{thm:prod-surjective} and \ref{thm:mini-surjective}).

\subsection{Proofs} \label{sec:proofs}

We first prove that inter-theoretically responsive metanormative theories and context-surjective adjustment functions do not jointly satisfy the weak Pareto property. 

\begin{theorem}\label{thm:notweakPareto}
A metanormative theory $f$ and adjustment  function $\mathit{adj}$ does not jointly satisfy the weak Pareto property if $f$ is inter-theoretically responsive and $\mathit{adj}$ is context-surjective.
\end{theorem}
\begin{proof}
Let $a, b \in A$ be arbitrary actions. Let $(T, c)$ be any ethical framework such that for all theories $t$, $t(a) < t(b)$. Moreover, let $T$ be such that there exist theories $u$ and $v$ with $u(a) > v(b)$. By inter-theoretic responsiveness, there exists a credence profile $C$ such that $a \succ_f b$, where $\preceq_f\ = f(T, C)$.

Because $\mathit{adj}$ is context-surjective, for all $C^*$, there exists context $\mathit{con^*}$ such that $C^* = \mathit{adj}(A, c, \mathit{con^*})$. Specifically, let $\mathit{con}$ be such that $C = \mathit{adj}(A, c, \mathit{con})$. Therefore, we have that, even though for all theories $t$, $t(a) < t(b)$, we also have that $a \succ_f b$, thereby violating the weak Pareto property.
\end{proof}

To show the significance of the above result, we first prove that MEC is inter-theoretically responsive.

\begin{theorem} \label{thm:mec-intertheoretic}
The metanormative theory MEC is inter-theoretically responsive.
\end{theorem}
\begin{proof}
Let $a, b$ be arbitrary actions. Let $T$ be any set of ethical theories such that there exist theories $u$ and $v$ with $u(a) > v(b)$.  We now present a credence profile $C$ such that $a \succ_\mathit{mec} b$ holds, where $\preceq_\mathit{mec}\ = f(T, C)$.

Under MEC, the actions are ordered by their weighted arithmetic mean. That is, for all credence profiles $C$:
\[\mathit{wam}(T, C, a) = \sum_{t \in T}c_a(t)t(a)\]
and
\[\mathit{wam}(T, C, b) = \sum_{t \in T}c_b(t)t(b)\]

By Lemma \ref{lemma:sum} (see later), we can set $c_a$ such that for any $\epsilon > 0$, we have $|u(a) - \mathit{wam}(T, C, a)| < \epsilon$, i.e. $\mathit{wam}(T, C, a) \in (u(a) - \epsilon, u(a) + \epsilon)$. Specifically, such that
\begin{equation} \label{eq:lower}
\mathit{wam}(T, C, a) > u(a) - \epsilon
\end{equation}

Similarly, by Lemma \ref{lemma:sum}, we can set $c_b$ such that for any $\delta > 0$, we have $|v(b) - \mathit{wam}(T, C, b)| < \delta$, i.e. $\mathit{wam}(T, C, b) \in (v(b) - \delta, v(b) + \delta)$. Specifically, such that

\begin{equation} \label{eq:upper}
\mathit{wam}(T, C, b) < v(b) + \delta
\end{equation}

We now show that for appropriate $\epsilon, \delta > 0$, $\mathit{wam}(T, C, a) > \mathit{wam}(T, C, b)$, i.e. $a \succ_\mathit{mec} b$. We know that $u(a) > v(b)$. Let $d = u(a) -v(b) > 0$. Then, let $\epsilon, \delta < \frac{d}{2}$.

To show that $\mathit{wam}(T, C, a) > \mathit{wam}(T, C, b)$, it's sufficient to show that $v(b) + \delta < u(a) - \epsilon$ (by Inequalities \ref{eq:lower} and \ref{eq:upper}). We can rearrange this and obtain that we need to show that $\delta + \epsilon < u(a) - v(b)$. Note that $u(a) - v(b) = d$, by definition.  Using the fact that $\epsilon, \delta < \frac{d}{2}$, we obtain that $\delta + \epsilon < \frac{d}{2} + \frac{d}{2} = d$, as required.
\end{proof}

We now prove that $\mathit{prod}$ satisfies context surjectivity.

\begin{theorem} \label{thm:prod-surjective}
The adjustment function $\mathit{prod}$ is context-surjective.
\end{theorem}
\begin{proof}
An adjustment function $\mathit{adj}$ is \emph{context-surjective} if for every set of actions $A = \{a,b,\ldots\}$, credence function $c$ and credence profile $C = \langle c_a, c_b ...\rangle$, there exists a context $\mathit{con} = \langle g_1, ..., g_m \rangle$ such that $ C' = C$, where $C' =  \mathit{adj}(A,c, \mathit{con}) = \langle c_a', c_b' ...\rangle$.

Specifically, we show that for arbitrary $x \in A$, we can set the context $\mathit{con}$ so that $c'_x = c_x$ holds.
For $\mathit{prod}$, given set of ethical theories $T$, for arbitrary $x \in A$, let
\begin{equation} \label{eq:prod:M}
M_x' = \min_{t \in T} \frac{c(t)}{c_x(t)}
\end{equation}
Note that for all $t \in T$, $c_x(t) > 0$ and $c_x(t) > 0$ (as they are both credence functions) and so $M_x' > 0$. Moreover, let $M_x = \min(1, M_x')$.

We can now define the context $\mathit{con}$. For any $x \in A$, $t \in T$, let:
\begin{equation} \label{eq:prod:g1}
g_1(x, t) = M_x \times \frac{c_x(t)}{c(t)}
\end{equation}
Note that we must have $g_1(x, t) \in (0, 1]$.  By definition of $M_x$ and $M_x'$, we know that $M_x \le M'_x \le \frac{c(t)}{c_x(t)}$. And so, we must have that $\frac{1}{M_x} \ge \frac{c_x(t)}{c(t)}$ and so $1 \ge  M_x \times \frac{c_x(t)}{c(t)}$, i.e.  $g_1(x, t) \le 1$. Moreover, we know that $M_x > 0$ and $\frac{c_x(t)}{c(t)} > 0$ and so $g_1(x, t) \ge 0$ is also true.

Moreover, for each $l \in [2, m]$, $x \in A$, $t \in T$, let
\begin{equation} \label{eq:prod:gl}
g_l(x, t) = 1
\end{equation}
By definition of $\mathit{prod}$, for each $t \in T$ and $x \in A$, it must be the case that:
\begin{equation}\label{eq:c-prime-x}
     c'_x(t) = \alpha \times g_1(x, t) \times \left[\prod_{l \in [2, m]}g_l(x, t)\right] c(t)
\end{equation}
 By Equations \ref{eq:prod:g1} and \ref{eq:prod:gl}, we substitute for $g_1$ and $g_l$ in Eq.\ref{eq:c-prime-x}:
\begin{equation}\label{eq:c-prime-x-sub}
    c'_x(t) = \alpha \times M_x \times \frac{c_x(t)}{c(t)} \times \left[\prod_{l \in [2, m]}1\right] c(t)\end{equation}
\[= \alpha \times M_x \times \frac{c_x(t)}{c(t)} \times  c(t)\]
\begin{equation} \label{eq:prod:c'}
= \alpha \times M_x \times c_x(t)
\end{equation}

We can calculate $\alpha$ from the fact that $c'_x$ must be unitary:
\[\sum_{t \in T}c'_x(t) = 1\]
Substituting in Equation \ref{eq:prod:c'} obtains
\[\sum_{t \in T}\alpha \times M_x \times c_x(t) = 1\]
rewritten as:
\[\alpha \times M_x \sum_{t \in T} c_x(t) = 1\]
Since $c_x$ is unitary ($\sum_{t \in T}c_x(t) = 1$):
\[\alpha \times M_x = 1\]
 which,   substituting in Equation \ref{eq:prod:c'}  obtains that for all $t \in T$, $x \in A$:
\[c'_x(t) 
=c_x(t)\]
Therefore, $\mathit{prod}$ is context-surjective.

\end{proof}

Similarly, we show that $\mathit{mini}$ is also context-surjective.

\begin{theorem} \label{thm:mini-surjective}
The adjustment function $\mathit{mini}$ is context-surjective.
\end{theorem}
\begin{proof}
An adjustment function $\mathit{adj}$ is \emph{context-surjective} if for every set of actions $A = \{a,b,\ldots\}$, credence function $c$ and credence profile $C = \langle c_a, c_b ...\rangle$, there exists a context $\mathit{con} = \langle g_1, ..., g_m \rangle$ such that $ C' = C$, where $C' =  \mathit{adj}(A,c, \mathit{con}) = \langle c_a', c_b' ...\rangle$.

Specifically, we show that for arbitrary $x \in A$, we can set the context $\mathit{con}$ so that $c'_x = c_x$ holds.
Let:
\begin{equation} \label{eq:mini:M}
M = \min_{t \in T} c(t)
\end{equation}
Note that for all $t$, $c(t) \in (0, 1)$, and so $M \in (0, 1)$.

We can now define the context $con$. For any $x \in A$, $t \in T$, let:
\begin{equation} \label{eq:mini:g1}
g_1(x, t) = M \times c_x(t)
\end{equation}
Note that $g_1$ is a contextual feature: we have $g_1(x, t) \in (0, 1]$ as $M \in (0, 1)$ and $c_x(t) \in (0, 1)$ and so their product is also between $0$ and $1$. Moreover, since $c_x(t) < 1$, then $g_1(x, t) \le M$.
Moreover, for any $l \in [2, m]$, any $x \in A$, $t \in T$, let
\begin{equation} \label{eq:mini:gl}
g_l(x, t) = 1
\end{equation}
By definition of $\mathit{mini}$, for any $t \in T$:
\begin{equation}\label{eq:c-prime-x-sub2}c'_x(t) = \alpha \min\left( g_1(x, t), \left[\min_{l \in [2, m]}g_l(x, t)\right], c(t)\right)\end{equation}
which, when substituting for $g_1(x,t)$ (Eq.\ref{eq:mini:g1}) and $g_1(x,t)$ (Eq.\ref{eq:mini:gl}) obtains:
\[c'_x(t) = \alpha \min\left(  M \times c_x(t), \left[\min_{l \in [2, m]}1\right], c(t)\right)\]
\[= \alpha \min\left(  M \times c_x(t), 1, c(t)\right)\]
and since $M \times c_x(t) < 1$ (recall Eq.\ref{eq:mini:g1} and
$g_1(x, t) \in (0, 1]$) and $c(t) < 1$:
\begin{equation}\label{Eq:Bla}c'_x(t) = \alpha \min\left(  M \times c_x(t), c(t)\right)\end{equation} 

From  Eq.\ref{eq:mini:M}, we have that $M \le c(t)$. Moreover, $c_x(t) < 1$ for all $t$. Hence $M \times c_x(t) < c(t)$, and so:
\begin{equation} \label{eq:mini:c'}
c'_x(t) = \alpha \times M \times c_x(t)
\end{equation}

We can calculate $\alpha$ from the fact that $c'_x$ must be unitary:
\[\sum_{t \in T}c'_x(t) = 1\]
Substituting in Equation \ref{eq:mini:c'}, we obtain:
\[\sum_{t \in T}\alpha \times M \times c_x(t) = 1\]
rewritten as:
\[ \alpha \times M \sum_{t \in T} c_x(t) = 1\]
Since $c_x$ is unitary ($\sum_{t \in T}c_x(t) = 1$), we obtain:
\[\alpha \times M = 1\]
And so 
 substituting  into Equation \ref{eq:mini:c'} we obtain that for all $t \in T$, $x \in A$:
\[c'_x(t) =  c_x(t)\] 
Therefore, $\mathit{mini}$ is context-surjective.
%$\alpha = \frac{1}{M}$ must be true (remark that $M \ne 0$). Therefore, we can substitute this into Equation \ref{eq:mini:c'} and obtain that for all $t \in T$, $x \in A$:
%\[c'_x(t) = \alpha \times M \times c_x(t)\]
%\[= \frac{1}{M} \times M \times c_x(t)\]
%\[= c_x(t)\]
%Therefore, $\mathit{mini}$ is context-surjective.

\end{proof}

From Theorems \ref{thm:notweakPareto}, \ref{thm:mec-intertheoretic}, \ref{thm:prod-surjective} and \ref{thm:mini-surjective} the following result immediately follows:

\begin{corollary} \label{thm:mec-prod-mini-notweakPareto}
The metanormative theory $\mathit{mec}$ with 1) the adjustment function $\mathit{prod}$ and 2) the adjustment function $\mathit{mini}$ jointly violate the weak Pareto principle.
\end{corollary}
We now prove the earlier referenced Lemma \ref{lemma:sum}:
%[I am planning on including the below lemma in the appendix because I think it has been probably proved in many math text books (but I couldn't find it myself, so I just proved it myself).]
\begin{lemma} \label{lemma:sum}
Let $S $ be a set of real numbers and let ${s_{i'} \in S}$ be an arbitrary element of $S$. Then, for any $\epsilon > 0$, there exists a weight function $w: S \to (0, 1)$ such that $\left| s_{i'} - \sum_{s_{i} \in S} w(s_{i})s_{i} \right| < \epsilon$.
\end{lemma}
\begin{proof}
We show our proof by constructing an appropriate $w$. In particular, let $w(s_{i'}) = 1 - \delta$ and for all $i \ne i'$, let ${w(s_i) = \frac{\delta}{|S| -1}}$ where $\delta > 0$ is a positive constant (we give its precise value later). Therefore, we must have $\sum_{s_i \in S \wedge i \neq i'} w(i) = \delta$ since we add up $\delta$ $|S| - 1$ times. Overall, we have:
\[\sum_{s_i \in S} w(s_i) s_i = (1-\delta)s_{i'} + \sum_{s_i \in S \wedge i \neq i'} w(s_i) s_i\]

And so, we must have:
\[\left| s_{i'} -\sum_{s_i \in S} w(s_i) s_i \right| = \left| s_{i'} -(1-\delta)s_{i'} - \sum_{s_i \in S \wedge i \neq i'} w(s_i) s_i \right|\]
\[= \left|\delta s_{i'} - \sum_{s_i \in S \wedge i \neq i'} w(s_i) s_i\right|\]

Since $\delta > 0$ and $w(s_i) > 0$ for all $s_i \in S$, we can derive the following inequality:
\[\left|\delta s_{i'} - \sum_{s_i \in S \wedge i \neq i'} w(s_i) s_i\right| \le |\delta s_{i'}| + \left| \sum_{s_i \in S \wedge i \neq i'} w(s_i) s_i\right| \]
\[\le \delta |s_{i'}| + \sum_{s_i \in S \wedge i \neq i'} w(s_i) |s_i|\]

Now, let $M' = \max_{s_i \in S} |s_i|$ and let $M = \max(1, M')$. Therefore, we must have  $\delta |s_{i'}| \le \delta M$. Moreover, for any other element $s_i$, it must be  that $w(s_i) |s_i| \le w(s_i) M$.
Hence:
\[\delta |s_{i'}| + \sum_{s_i \in S \wedge i \neq i'} w(s_i) |s_i| \le \delta M + \sum_{s_i \in S \wedge i \neq i'} w(s_i) M\]
\[= \delta M + M \sum_{s_i \in S \wedge i \neq i'} w(s_i)\]
Note that we have defined $w$ such that $\sum_{s_i \in S \wedge i \neq i'} w(i) = \delta$, and so we must have
\[= \delta M + M \delta = 2\delta M\]

Note that $M > 0$ is some positive constant. Therefore, we can choose $\delta$ to be arbitrarily small, namely, we can make $2\delta M < \epsilon$ for any chosen $\epsilon > 0$ by setting $\delta < \frac{\epsilon}{2M}$.
\end{proof}

\subsection{Simpson's Paradox} \label{sec:simpson}
We now invoke \emph{Simpson's paradox} by way of commenting on what at first glance is the seemingly unintuitive violation of the weak Pareto property by context-adjusted metanormative theories.
A well-known illustration of Simpson's paradox is the gender disparity paradox observed in admissions to the University of California, Berkeley \cite{wagner1982simpson}.  The fact that men were more likely to be admitted to Berkeley than women, prompted the university to review each department's admission rates. They discovered a seeming  paradox: for each department, the admission rate for women was higher than for men. The explanation was that women were more likely than men to apply to departments whose admission rates were lower, with  men more likely to apply to departments with higher admissions rates. As a result, while women had higher admission rates for individual departments,  overall  admission rates  were lower.

Note that the men had higher acceptance rates because the acceptance rates of women in the `harder-to-get-into' departments was lower than the acceptance rates of men in the `easier-to-get-into' departments. In other words, for the average acceptance rate of men to be higher than that of women, there had to be two departments, such that the average acceptance rate of men in an `easier' department was higher than the average acceptance rate of women in a `harder' department. This idea is captured by inter-theoretic responsiveness: if under one theory $u$ an action $a$ is preferred to another theory $v$ for action $b$ (i.e. $u(a) > v(b)$) then we can set the credence such that $a$ wins out over $b$. In other words, inter-theoretic responsiveness is the systematic possibility of Simpson's paradox happening.

\begin{table}[t]
    \centering
    \begin{tabular}{ccccc}
         \hline
         & \multicolumn{2}{c}{\textbf{Left room} $l$} & \multicolumn{2}{c}{\textbf{Right room} $r$} \\
         & $c_l(t)$& $t(l)$& $c_r(t)$& $t(r)$\\ \hline
         Act utilitarianism $\mathbf{u}$& $0.6$& $1$& $0.13$& $0$\\
         Two-level utilitarianism $\mathbf{v}$& $0.4$& $4$& $0.87$& $3$\\ \hline
         $\mathit{wam}$& \multicolumn{2}{c}{$2.2$} & \multicolumn{2}{c}{$\underline{2.61}$}\\
    \end{tabular}
    \caption{FROBO example demonstrating the violation of the weak Pareto property using MEC.} %Good enough explanation?
    \label{table:Simpson}
\end{table}

Let us now consider a concrete example demonstrating Simpson's paradox for MEC and $\mathit{prod}$. Let $\mathbf{u}$ be an (act) utilitarian ethical theory with $\mathbf{u}(l) = 1$ and $\mathbf{u}(r) = 0$. Let the other theory be $\mathbf{v}$, a two-level utilitarian theory. In particular, the rule-utilitarian component of $\mathbf{v}$ has a rule that FROBO should save lives wherever possible. As a result, attending to the right room is seen as quite a good action, as it satisfies an ethical norm, i.e. $\mathbf{v}(r) = 3$. However, FROBO has no established ethical rules to evaluate  the left room; FROBO resorts to evaluating the left room from first (utilitarian) principles. As stated earlier, FROBO reasons that the blockade will not be lifted and so the left room is a preferable action to the right room, i.e. $\mathbf{v}(l) = 4$.

Assume that the only contextual feature FROBO takes into account is the resource bounds $rb$. Here, consider   a modified version of the FROBO example wherein the computational power of FROBO is severely limited. We therefore  assume that FROBO's (act) utilitarian calculations do not fare well under these stricter resource bounds in either of the rooms, i.e. $rb(l, \mathbf{u}) = rb(r, \mathbf{u}) = 0.1$. Since $\mathbf{v}$'s evaluations of the left room are based on similar calculations, it also does not perform well under FROBO's resource constraints, i.e. $rb(l, \mathbf{v}) = 0.1$. However, the ethical rule FROBO used to evaluate the right room is not affected by these constraints, i.e. $rb(r, \mathbf{v}) = 1$. In other words, $\mathbf{v}$ is an ethical theory such that only one of its prescriptions are negatively impacted by the resource bounds.

Assume the initial credences are $c(\mathbf{u}) = 0.6$ and $c(\mathbf{v}) = 0.4$. Using $\mathit{prod}$, we can calculate the updated credences. First, $c_l$ is such that $c_l(\mathbf{u}) = \alpha \times 0.1 \times 0.6 = 0.06 \alpha$ and $c_l(\mathbf{v}) = \alpha \times 0.1 \times 0.4 = 0.04 \alpha$, Therefore, $\alpha = 10$ and so $c_l(\mathbf{u}) = 0.6$ and $c_l(\mathbf{v}) = 0.4$. In other words, for the left room the initial credences remain unchanged, i.e. $c_l = c$.

Second, $c_r$ is such that $c_r(\mathbf{u}) = \beta \times 0.1 \times 0.6 = 0.06 \beta$ and $c_r(\mathbf{v}) = \beta \times 1 \times 0.4 = 0.4 \beta$. Therefore, $\beta = \frac{50}{23}$ and $c_r(\mathbf{u}) \approx 0.13$ and $c_r(\mathbf{v}) \approx  0.87$. In other words, for the right room, two-level utilitarian $\mathbf{v}$ is weighed significantly more, despite its lower initial credence.

We calculate the weighted arithmetic mean of both actions. For the left room:
$\mathit{wam}(\{ \mathbf{u}, \mathbf{v}\}, \langle c_l, c_r \rangle, l) = c_l(\mathbf{u}) \times \mathbf{u}(l) + c_l(\mathbf{v}) \times \mathbf{v}(l) = 0.6 \times 1 + 0.4 \times 4 = 2.2$. 
For the right room:
$\mathit{wam}(\{\mathbf{u}, \mathbf{v}\}, \langle c_l, c_r \rangle, r) = c_r(\mathbf{u}) \times \mathbf{u}(r) + c_r(\mathbf{v}) \times \mathbf{v}(r) \approx 0.13 \times 0 + 0.87 \times 3 \approx 2.61$.
Due to its higher weighted arithmetic mean, MEC prefers the right room.

That is, despite all ethical theories preferring the left room to the right room, overall the right room is preferred. Note that this happened because for the left room the evaluation of $\mathbf{u}$ was dominant, while for the right room the evaluation of $\mathbf{v}$ was dominant, where $\mathbf{u}(l) < \mathbf{v}(r)$.
\section{Conclusions and Future Work}\label{Sec:Discussion}
The central thesis of our paper is that contextual factors should be taken into account when aggregating ethical evaluations of actions elicited from simulations. We believe that our formalisation of context-adjusted moral credences in the ethical theories licensing evaluations, and the study thereof,
opens up a new area of study within moral uncertainty research and value-alignment more generally, with many crucial follow up questions.

\paragraph{Contextually-Weighted Social Choice}
We have shown that MEC - the standard mode of aggregation in moral uncertainty - may violate the weak Pareto property. Moreover, we have argued that despite the apparent paradox,
this is not problematic. This departure from standard assumptions implies that contextually-weighted preference aggregation diverges significantly from traditional social choice theory, where the weak Pareto property is typically regarded as foundational \cite{sep-social-choice}.  Our findings therefore motivate a reassessment of the criteria used to evaluate aggregation methods in context-sensitive settings. Addressing these questions could lead to a new area of research at the intersection of moral uncertainty and social choice theory.

\paragraph{Contextual Factors}
We have identified several contextual factors that influence how pluralistic ethical evaluations are aggregated. We do not claim this list to be exhaustive. An important direction for future work is to more comprehensively investigate additional factors that may shape aggregation. For example, the extent to which unanticipated context-specific alternative choices of action may impact the degree to which one is risk averse \cite{Buchak}, and how this may differentially impact credences in moral theories.

\paragraph{Moral Uncertainty}
This paper has not addressed several key challenges in the literature on moral uncertainty. One such challenge is the \emph{problem of fanaticism} \cite{macaskill2020moral,szabo2024moral}, in which low-credence moral theories can disproportionately influence an agent's decision-making. Notably, MEC is known to be susceptible to this problem. Further research is needed to understand how contextual factors might mitigate or exacerbate fanaticism and related phenomena.

\paragraph{Adjustment Functions}
We introduced two simple adjustment functions: $\mathit{prod}$ and $\mathit{mini}$. However, we do not claim that these functions are \emph{descriptive} -- that is, reflective of how humans actually incorporate contextual information -- nor do we claim that they are \emph{prescriptive} -- that agents ought to adjust their credence in these ways. Future work should explore adjustment functions along both these dimensions. Empirical studies could help illuminate the descriptive question, while philosophical analysis can contribute to understanding the normative dimension.
%Considering things like being responsive to stakes and credences fanaticism

\paragraph{`On the fly' Alignment Through Dialogue}
In Section \ref{Sec:ContextualFactors} we suggested application of our work to scenarios in which an \textit{individual}’s relative credences in distinct ethical theories is used to support AI decision making ‘on
the fly'. Indeed, while the value alignment problem initially came to prominence
in anticipation of more generally intelligent, in particular `superintelligent', systems \cite{bostrom2014superintelligence}, it also applies more prosaically to narrow AI systems. 

Consider a personal assistant large language model PAL who assembles holiday itineraries for Sally. Based on her basic requirements, PAL recommends  itineraries in Malaga  and Lanzarote.  Sally then asks PAL to estimate the overall carbon footprints of these   itineraries. On this basis Sally adopts a consequentialist argument for preferring Malaga. However PAL reminds Sally that in their previous interactions, she deontologically preferred destinations with better value-for-money hotels, and that this constitutes a deontological argument for preferring Lanzarote. Sally responds by arguing that in this decision making context  sustainability is more important than cost, because she will be spending very little time in a hotel, given that the weather in both destinations will be gorgeous. In other words, in this context her moral credence for the consequentialist based preference is greater than the deontology based preference. As a result of this interaction, PAL's understanding of Sally's values is augmented and learnt for re-use in future interactions. 

This scenario adheres to the spirit of value alignment solutions advocated by \cite{CIRL,russell2019human}, viz.a.vie. that AI systems should  perform tasks while simultaneously learning users' value-based preferences as they evolve over time. However, these works assume that humans know their value based preferences from the outset (and are learnt by passive observation and instruction from users), rather than being shaped by reasons and argument. On the other hand, in the Sally-PAL dialogical interaction, PAL’s superior information processing is leveraged in support of the decision making task, while also exploiting PAL’s superior consequential reasoning in helping Sally establish her value-based preferences. Indeed, she is further supported by PAL’s reference to her prior decisions, which help establish -- within this particular decision-making context -- differential moral credences assigned to the arguments associated with distinct consequentialist and deontological theories. We therefore aim to integrate our work on moral credences and the impact of contextual factors, with ongoing proposals for argumentation-based dialogues designed to support value-alignment \cite{bezouvrakatseli2024dialoguesjointhumanaireasoning}. Moreover, the uses of argument to support users in assigning theory specific valuations of actions and their differential credences in these theories, would be especially useful when subjects are asked to judge simulations in the multi-step pipeline approaches our current paper assumes \cite{DBLP:journals/ibmrd/NoothigattuBMCM19}. After all, the simulated scenarios typically present the kinds of ethical dilemmas that subjects are unlikely to have encountered in their everyday lives. As a result: 1) they are unlikely to feel confident in their ethical evaluations; 2) \textit{individual} subjects may wish to assess \textit{individual} actions under different ethical theories, attributing varying degrees of confidence (i.e. moral credences) to the relevance of each theory, and 3) they may seek arguments that offer \textit{prescriptive} guidance, especially since eliciting their \textit{descriptive} preferences -- typically used for value alignment -- is less feasible given the ethical novelty of the simulations.

\paragraph{Thick Ethics}
Moral uncertainty research, and more generally value-alignment research,  often makes strong, oversimplifying assumption regarding ethical theories; in philosophical terms, ethical theories are represented `thinly' \cite{sep-thick-ethical-concepts,macaskill2020moral}. That is, moral theories are merely seen as utility functions or preference orderings. Such representations ignore the `thick' concepts/commitments these ethical theories subscribe to: \emph{how} and \emph{why} these distinct theories support ethical evaluations. Ignoring the `how' is problematic as many ethical theories have fundamentally different algorithmic properties, as discussed in this paper. Moreover, ignoring the `why' is problematic as ethical theories, such as deontological ethical theories, cannot be properly represented by utility functions or preference orderings \cite{sep-ethics-deontological,macaskill2020moral}. In short, the very need to consider contextual factors arises from the limitations of thin representations of moral theories. In this sense, our paper can be viewed as an attempt to `thicken' ethical evaluations derived from simulations. An alternative -- and potentially complementary -- approach would be to account for inherently thicker rationales for ethical evaluations such as those elicited by dialogue and argument. The Sally-PAL  dialogue illustrates extraction of such richer ethical information than that captured through preferences or utility functions alone.

\bigskip
\bibliography{mybibfile}

\end{document}